\documentclass{article}

\usepackage{arxiv}

\usepackage[utf8]{inputenc} 
\usepackage[T1]{fontenc}    
\usepackage{hyperref}       
\usepackage{url}            
\usepackage{booktabs}       
\usepackage{amsfonts}       
\usepackage{nicefrac}       
\usepackage{microtype}      
\usepackage{lipsum}		
\usepackage{refcount} 
\usepackage{graphicx}
\usepackage{natbib}
\usepackage{doi}
\usepackage{comment}
\usepackage{times}
\usepackage{soul}
\usepackage{url}
\usepackage[small]{caption}
\usepackage{graphicx}
\usepackage{amsmath}
\usepackage{amsthm}
\usepackage{booktabs}
\usepackage{algorithm}
\usepackage{algorithmic}
\usepackage[switch]{lineno}
\usepackage{amssymb}

\newtheorem{theorem}{Theorem}
\newtheorem{definition}{Definition}
\newtheorem{proposition}{Proposition}
\newtheorem{lemma}{Lemma}

\newcommand{\Pm}{\pm} 

\title{
Shapley Regression
for Rare Disease Diagnosis Support: a case study on APDS
}

\author{ \href{https://orcid.org/0000-0002-4132-1068}{\includegraphics[scale=0.06]{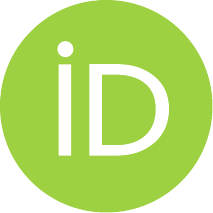}\hspace{1mm}Safa Alsaidi} \\
	Inria, Inserm, UPC,\\HeKA U1346, Paris, France. \\
	\texttt{safa.alsaidi@inria.fr} 
	\And
	\href{https://orcid.org/0009-0001-9364-8572}{\includegraphics[scale=0.06]{orcid.pdf}\hspace{1mm}Tomás Brogueira} \\
	Técnico, University of Lisbon, \\ INESC-ID, Lisbon, Portugal\\
\texttt{tomas.brogueira@tecnico.ulisboa.pt} \\
	\And 
	\href{https://orcid.org/0000-0002-0030-8094}{\includegraphics[scale=0.06]{orcid.pdf}\hspace{1mm}Nizar Mahlaoui} \\
	 Necker Enfants Malades University Hospital,\\ AP-HP, Paris, France\\
	\texttt{nizar.mahlaoui@aphp.fr} \\
	 \And 
	\href{https://orcid.org/0000-0001-6921-8161}{\includegraphics[scale=0.06]{orcid.pdf}\hspace{1mm}Marc Vincent}\\
	  Data Science Platform, INSERM UMR1163, Imagine Institute,\\ UPC, Paris, France\\
\texttt{marc.vincent@institutimagine.org,} \\
	  \And	 
	\href{https://orcid.org/0000-0001-7301-6167}{\includegraphics[scale=0.06]{orcid.pdf}\hspace{1mm}Guilherme Pelegrina}\\
    Mackenzie Presbyterian University,\\ São Paulo, Brazil\\   \texttt{guilherme.pelegrina@mackenzie.br} \\
      \And 
	\href{https://orcid.org/0000-0002-3326-2811}{\includegraphics[scale=0.06]{orcid.pdf}\hspace{1mm}Nicolas Garcelon }\\
    Data Science Platform, INSERM UMR1163, Imagine Institute\\ UPC, Paris, France\\
\texttt{nicolas.garcelon@institutimagine.org} \\
	 \And	 
	\href{https://orcid.org/0000-0002-1466-062X}{\includegraphics[scale=0.06]{orcid.pdf}\hspace{1mm}Adrien Coulet} \\
	 Inria, Inserm, UPC,\\
     HeKA U1346, Paris, France. \\
	 \texttt{ adrien.coulet@inria.fr} \\
      \And 
	\href{https://orcid.org/0000-0003-2316-7623}{\includegraphics[scale=0.06]{orcid.pdf}\hspace{1mm}Miguel Couceiro}\\
	Técnico, University of Lisbon, \\ INESC-ID, Lisbon, Portugal\\
	 \texttt{miguel.couceiro@inesc-id.pt} \\
}

\begin{document}

\maketitle

\begin{abstract}

Activated PI3K$\delta$ Syndrome (APDS) is a rare genetic immune disorder caused by variants in PIK3CD or PIK3R1, with highly heterogeneous symptoms that often delay diagnosis. Early recognition is hampered by overlapping clinical presentations and limited clinician awareness, motivating systematic, data-driven approaches to detect APDS-associated phenotypic patterns in routine electronic health records. Traditional linear scoring systems cannot capture complex symptom interactions, while deep learning models, though expressive, often lack interpretability. To bridge this gap, we propose {\it Shapley regression}, a novel game-theoretic model replacing the linear predictor with a $k$-additive cooperative game, explicitly modeling co-occurrence of symptoms while maintaining the transparency and convexity of logistic regression. We carry out an empirical study of our lightweight method on eight public biomedical datasets, showing that a 2-additive model with $\ell_2$ regularization achieves an optimal trade-off between predictive power and noise robustness. We also apply it to a real-world cohort of 222 patients, on which Shapley regression accurately distinguished APDS cases from matched controls, confirming and validating phenotypes known to be associated with APDS, and facilitating the exploration of pairwise interactions between symptoms, validated by clinical experts.\footnote{
 The main contributions of this paper will be presented at IJCAI 2026. The first two named authors had equal contribution. }

\end{abstract}

\section{Introduction}

Rare primary immunodeficiencies pose a substantial diagnostic challenge due to their low prevalence, variable symptom presentation, and overlap with more common immune disorders \cite{faviez_diagnosis_2020}. Activated PI3K$\delta$ Syndrome (APDS) exemplifies these difficulties. Caused by pathogenic variants in the PIK3CD or PIK3R1 genes, APDS affects an estimated 1--2 individuals per million \cite{lougaris_activated_2024,thouenon_activated_2021}. However, its true global prevalence is unknown and likely underestimated due to underdiagnosis and incomplete registry reporting \cite{maccari_activated_2023,mahlaoui_prevalence_2017}. Moreover, APDS was only genetically characterized in 2013 \cite{angulo_phosphoinositide_2013}. Confirmed cases demonstrate that APDS patients can develop recurrent infections, lymphoproliferation, autoimmunity, bronchiectasis, and an increased risk of malignancy; however, the specific combination and severity of symptoms vary widely among individuals \cite{zhu_activated_2024,singh_updated_2020}. Because many clinical manifestations overlap with other primary immunodeficiencies and immune-mediated conditions, early recognition is difficult and often delayed \cite{lougaris_activated_2024,tavakol_diagnostic_2020}. Moreover, limited awareness among healthcare providers frequently results in prolonged morbidity and missed opportunities for targeted interventions before a definitive genetic explanation is available \cite{vanselow_activated_2023}. Given this, it is likely that a substantial number of APDS cases remain undiagnosed worldwide, further obscuring the true disease burden and delaying access to appropriate genetic testing and care.

From an ML perspective, diagnosing rare diseases like APDS presents a specific ``small data, high complexity'' dilemma. On the one hand, traditional linear models such as standard Logistic Regression offer stable learning on small datasets but are inherently limited by their purely additive structure (unless one includes interaction terms manually); they assume that each symptom contributes independently and cannot capture situations in which the joint presence of two symptoms conveys more information than the sum of their individual effects.
On the other hand, high-capacity non-linear models (like neural networks) can capture these interactions but are prone to severe overfitting on small cohorts and lack the transparency required for clinical adoption.

To address this challenge, we propose \emph{Shapley regression}, a 
game-theoretic extension of 
logistic regression in which the affine form is replaced by a cooperative game. We provide the implementation on our github repository.\footnote{\url{https://github.com/tomasbrogueira/shapley_regression}}

\begin{enumerate}
  \item \textbf{Model Non-Linearity:} Our proposed approach can account for crucial interaction effects (synergy and redundancy) between features.
  \item \textbf{Maintain Interpretability:} We parameterize the model using Shapley Interaction Indices. This means the model's weights are not abstract values, but direct game-theoretic measures of how much each feature (or feature pair) contributes to the diagnosis.
  \item \textbf{Ensure Stability:} We utilize a $k$-additive formulation, limiting interactions to a specific order (e.g., $k=2$ for pairwise interactions). This drastically reduces the parameter space, preventing the curse of dimensionality.
  \item \textbf{Theoretical guarantees and explanations by design:} Our 
  lightweight solution can account for the optimal trade-off between model performance and complexity. Moreover, its parameters correspond to importance criteria and their usefulness is attested on the medical domain, through an empirical study on the detection of APDS diganosis.
\end{enumerate}

We performed an extensive validation study to determine the most appropriate architecture for medical use. As detailed Supplementary Material ({Appendix~\ref{app:technical_validation}}), we benchmarked the method on synthetic data and eight public health datasets (including COVID-19 and Diabetes). These experiments confirmed that a 2-additive model with Ridge ($\ell_2$) regularization maximizes stability against noise while effectively capturing biological interactions, justifying its use in this study.

We then evaluated our validated framework on a real-world cohort of APDS patients. Our objective is twofold: firstly, to enable accurate detection of APDS manifestations; secondly, to exploit the model's game-theoretic structure to map the  landscape of symptom interactions, providing interpretable insights into the phenotypic patterns and symptom combinations that are prevalent across patients.

\section{Related Work}

Rare diseases collectively affect hundreds of millions of individuals worldwide, yet they remain vastly underdiagnosed, often taking years to be correctly identified. This diagnostic delay is largely attributable to their low prevalence, marked clinical heterogeneity, and substantial symptom overlap with more common conditions \cite{phillips_time_2024,alves_computer-assisted_2016}. In recent years, machine learning (ML) and artificial intelligence (AI) have been increasingly explored as potential tools to address these challenges \cite{topol_deep_2019,rajkomar_machine_2019}. A scoping review published in 2020 identified more than 200 studies applying ML methods to 74 different rare diseases, with a strong emphasis on diagnostic and prognostic tasks \cite{schaefer_use_2020}. However, the review also highlighted persistent methodological limitations, including small sample sizes, severe class imbalance, and a lack of external validation or systematic comparison with clinical expertise, underscoring the need for more robust and clinically grounded approaches.
Despite the increasing interest in computational diagnosis of rare diseases generally, there is limited published work specifically addressing ML based classification of primary immunodeficiency disorders such as APDS. Current diagnostic practice primarily relies on clinical evaluation combined with genetic testing, while only a limited number of proof-of-concept studies in related primary immunodeficiency disorders have explored the use of ML models and feature-importance analyses to support differential diagnosis \cite{mendez_barrera_whos_2023}.

A fundamental obstacle in applying ML to rare diseases such as APDS is the scarcity and imbalance of available data. Many rare conditions are represented by fewer than 100 documented cases in electronic health records (EHRs), making it difficult to train models that generalize beyond the training cohort \cite{mitani_small_2020}. To mitigate these challenges, prior work has proposed pipelines that integrate phenotype extraction, semantic similarity measures, and imbalance-aware learning strategies to amplify informative clinical signals while reducing noise from dominant control populations \cite{faviez_performance_2024,patrick_enhanced_2022}. Equally critical is the issue of explainability and clinical credibility. High-performing black-box models often fail to provide transparent reasoning for their predictions, which limits clinician trust and hinders clinical adoption \cite{raposo_fifty_2025}. As a result, there has been increasing interest in interpretable modeling approaches that produce human-readable explanations, align with domain knowledge, and can be evaluated in conjunction with expert clinical review \cite{ronicke_can_2019,rider_evaluating_2025}. Such explainable approaches are particularly valuable in diseases like APDS, where diagnostic patterns are subtle and require specialized immunological expertise to interpret.

Within this context, capacities and their associated Choquet integrals offer a principled framework for modeling complex feature interactions while maintaining interpretability. These aggregation functions have been extensively studied in multi-criteria decision making (MCDM) \cite{ModelingDecisionsAggregation,AggregationFunctionsBook}, and their connections to cooperative game theory—particularly through the Shapley value and interaction indices—provide powerful tools for quantifying feature importance and interdependencies in a real-world setting.

The integration of Choquet integrals as an aggregation function for a classification problem, in specific for logistic regression, was first proposed by Tehrani et al. \cite{OriginalChoquistic}, who introduced ``Choquistic regression''. Their approach yielded a monotone, interaction-aware model that outperformed standard logistic regression while maintaining interpretability. More recently, Pelegrina et al. \cite{pelegrina2025gametheoretic} proposed a game-theoretic logistic regression framework that eliminated the need for monotonicity constraints while still capturing complex feature interactions. Various parameterizations and computational efficiency improvements have since been explored to address the practical challenges of learning these models \cite{CapacityBasedClassification,ChoquetIntegralSHAPValues}.

\section{Proposed Method}
To tackle the ``small data, high complexity'' challenge inherent to rare disease diagnosis, we propose \textbf{Shapley Regression (SR)}. This framework extends the standard logistic regression (LR) model by incorporating non-linear interactions derived from Cooperative Game Theory.

LR assumes that features contribute additively to a diagnosis\footnote{I.e., the risk of Feature A + Feature B is Risk A + Risk B.}, however, biological systems are rarely additive. Instead, they often exhibit \textit{synergistic} effects (\textit{e.g.}, two features jointly provide much stronger evidence for a diagnosis than either feature alone) or \textit{redundant} effects (\textit{e.g.}, two features convey overlapping information, such that the presence of one reduces the additional diagnostic value of the other).
To capture this while maintaining the interpretability of a linear model, we use cooperative games that can be parametrized by ``{Shapley Interaction Indices}'' \cite{Grabisch2016SetFunctions}. In this section we recall the necessary background of our proposed model and the subsequent theoretical analysis, and provide further details in the Suplementary Material (Appendix~\ref{sec:math-foundations}).
\subsection{Shapley Regression Formulation}
We consider a binary classification task with a feature vector $\mathbf{x}=(x_1,\dots,x_n)$ normalized to $[0,1]^n$. We define the probability of the positive class (APDS) as:
\begin{equation*}
P(y=1\mid \mathbf{x}) = \sigma\left(\beta + \sum_{A \subseteq \mathcal{F}, A \neq \emptyset} I(A) \cdot \phi_A(\mathbf{x}) \right)
\end{equation*}
where $\mathcal{F}=\{1,\dots,n\}$ is the set of clinical features, $\sigma(\cdot)$ is the sigmoid function, and $\beta$ is the bias. Here, the learnable coefficients $I(A)$ are exactly the \textbf{Shapley Interaction Indices} associated with a subset of features $A$. The basis functions $\phi_A(\mathbf{x})$ are constructed to allow this direct game-theoretic interpretation. For any coalition of features $A$, the basis function is defined as \cite{ChoquetIntegralSHAPValues}:
\begin{equation}\label{eq:shapley-agg}
\phi_A(\mathbf{x})= \sum_{C\subseteq A} \frac{(-1)^{|A|-|C|}}{|A|-|C|+1}\; \min_{i\in C}x_i.
\end{equation}
This formulation allows us to learn a non-linear decision boundary where the weights directly represent the marginal contribution of features (and feature combinations) to the positive class probability.

\subsection{2-Additive Complexity and Interpretation}
A fully general Choquet integral would require $2^n$ parameters, which is computationally intractable and prone to overfitting. To ensure stability on small medical cohorts, we enforce \textbf{$k$-additivity}, specifically setting $k=2$. This choice provides the optimal performance-complexity as empirically attested.
This restricts the model to learning only \textbf{main effects} (singletons, $|A|=1$) and \textbf{pairwise interactions} ($|A|=2$). Higher-order interactions are forced to zero. This reduces the parameter space from exponential to quadratic ($\mathcal{O}(n^2)$), specifically $n + \binom{n}{2}$ coefficients.
The resulting model offers a transparent interpretation for clinicians:
 \paragraph{\textbf{Main Effects ($I(\{i\})$):}} The independent contribution of clinical feature $i$ (analogous to standard LR coefficients).
    \paragraph{\textbf{Pairwise Interactions ($I(\{i,j\})$):}}
    \begin{itemize}
       \item {\it Synergy}: $I > 0$. The presence of \textit{both} features increases the probability of APDS more than the sum of their individual parts (complementarity).\\
       \item {\it Redundancy}: $ I < 0$. The features provide overlapping information; observing the second adds less diagnostic value if the first is already present.\\
       \item {\it Independence}: $I \approx 0$. The features act independently.
\end{itemize}

\subsection{Model Training}\label{sec:model-training}
To train the model, we construct a design matrix $\Phi \in \mathbb{R}^{N \times p}$ where each column corresponds to a basis function $\phi_A(\mathbf{x})$ for $|A| \leq 2$. The model is then trained as a standard logistic regression on this expanded feature space by minimizing the regularized negative log-likelihood:
\begin{equation*}\label{eq:loss}
\mathcal{L}(\boldsymbol{\theta},\beta)=
-\sum_{i=1}^{N}\Bigl[
y^{(i)}\ln \hat{p}^{(i)}+
\bigl(1-y^{(i)}\bigr)\ln\!\bigl(1-\hat{p}^{(i)}\bigr)\Bigr] + \lambda\|\mathbf{I}\|_{q}
\end{equation*}
where $\mathbf{I}$ is the vector of Shapley indices. We employ $\ell_2$ regularization ($q=2$, Ridge), which encourages smoothness and, as shown in the following section, guarantees the stability of the explanations.

\subsection{Theoretical Generalisation and Stability Analysis}\label{sec:theoretical-bounds}

To interpret the empirical performance of the $k$-additive Choquet model, we analyse the generalisation gap—the difference between expected and empirical risk—using Structural Risk Minimization (SRM). We examine three distinct complexity regimes. Our goal is to explain why regularised models (Regimes 2 and 3) succeed in high-order settings where unregularised models (Regime 1) fail.

\paragraph{Notation.}
Let $\mathcal{H}_k$ denote the hypothesis class of $k$-additive Choquet integrals defined with respect to a cooperative game $v$. The complexity of this class depends on the dimension $D_k = \sum_{j=1}^{k} \binom{n}{j}$. We seek to bound the generalisation gap $R(h) - \hat{R}_N(h)$,
where $R$ and $\hat{R}$ denote respectively the true and empirical risk (error rate).

\subsubsection{Validation Protocol for Theoretical Bounds.}
To empirically validate the generalisation bounds derived in the following sections, we conducted controlled experiments using \textbf{random synthetic data} ($\mathbf{X} \sim U[0,1]^n, y \sim \text{Bernoulli}(0.5)$).
Using pure noise ensures that our measurements capture the intrinsic \textbf{capacity} and \textbf{stability} of the hypothesis class $\mathcal{H}_k$ itself, independent of any signal in the data. This ``worst-case'' testing confirms that the bounds hold for \textit{any} data distribution. Unless otherwise stated, validation plots use $N=100$ samples and $n=10$ features.

\subsubsection{Regime 1: Unregularised (Baseline VC Bound)}
In the absence of constraints, the model's capacity is determined purely by the count of its parameters. The generalisation gap is bounded using the Vapnik-Chervonenkis (VC) dimension $d_{\text{VC}} = D_k + 1$:
\begin{equation*}
R(h) - \hat{R}_N(h) \le \mathcal{O}\left(\sqrt{\frac{D_k}{N}}\right).
\end{equation*}
\\
\textbf{Limitation:} This bound scales linearly with $\sqrt{D_k}$. Since $D_k$ grows polynomially with $n$ ($\mathcal{O}(n^k)$), this bound becomes loose rapidly. It predicts that without regularisation, increasing $k$ requires an exponentially larger dataset to prevent overfitting.

\subsubsection{Regime 2: \texorpdfstring{$\ell_1$}{l1} Regularisation (Rademacher Bound)}
When applying $\ell_1$ regularisation, we restrict the hypothesis class to functions with bounded norm $\|\boldsymbol{\theta}\|_1 \le B$. The VC bound is overly pessimistic here because it ignores this constraint. A tighter bound is provided by the \textit{Rademacher Complexity}, which measures the ability of the constrained class to fit random noise:
\begin{equation*}
R(h) - \hat{R}_N(h) \le \mathcal{O}\left(B \sqrt{\frac{\ln(D_k)}{N}}\right).
\end{equation*}

\noindent \textbf{Comparison:} The crucial improvement over the VC bound is that the dependence on dimension $D_k$ improves from \textbf{polynomial} ($\sqrt{D_k}$) to \textbf{logarithmic} ($\sqrt{\ln D_k}$).
This implies that $\ell_1$-regularised models are significantly less sensitive to the ``curse of dimensionality.'' Mathematically, this justifies why we can explore high-order interactions (large $k$) using Lasso: the cost of adding extra dimensions is negligible, provided the underlying true model is sparse.

\subsubsection{Regime 3: \texorpdfstring{$\ell_2$}{l2} Regularisation (Algorithmic Stability)}
The $\ell_2$ penalty (Ridge) induces strict convexity in the loss function. This allows us to move away from counting parameters (VC dimension) and instead analyze the \textit{Algorithmic Stability} of the optimiser. Stability bounds guarantee that if the algorithm is robust to small changes in the training set, it will generalize well.

\noindent \textbf{The Bound:} For a loss function with Lipschitz constant $L$ and regularisation strength $\lambda$, the generalisation gap is bounded by the stability coefficient $\beta$:
\begin{equation*}
\mathbb{E}[R(h) - \hat{R}_N(h)] \le \beta \le \frac{2 L^2}{\lambda N}.
\end{equation*}

Where the term $L$ represents the Lipschitz constant of the loss function with respect to the parameters. In our context, $L$ scales with the maximum Euclidean norm of the feature vectors, $L \approx \sup_x \|\Phi_k(\mathbf{x})\|_2$.

\noindent \textbf{Comparison:}
While the numerator $L^2$ does grow with $k$ (as feature vectors get longer), this bound offers a crucial control mechanism that VC theory lacks: the regularisation parameter $\lambda$ in the denominator.
\begin{itemize}
  \item 
  In the \textbf{Unregularised (VC)} regime, the complexity is fixed by the architecture ($\sqrt{D_k}$). Overfitting is inevitable as $k$ grows.
  \item 
  In the \textbf{$\ell_2$ Stability} regime, we can actively counteract the growth of the feature space by increasing $\lambda$.
\end{itemize}

\begin{figure}[t]
  \centering
\includegraphics[width=0.7\linewidth]{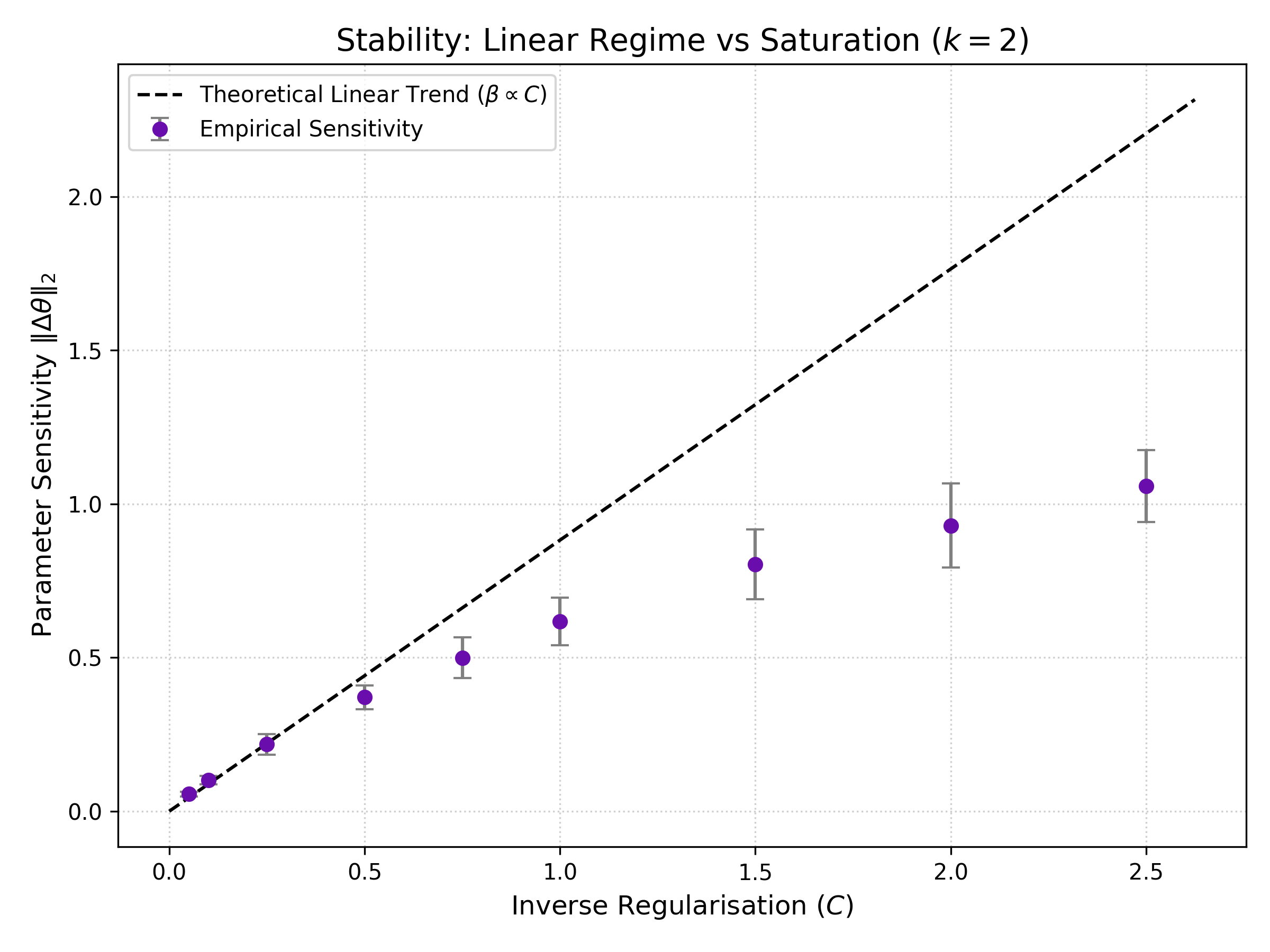}
  \caption{\textbf{Empirical Validation of Algorithmic Stability.}
  We measure 
  sensitivity (Euclidean shift $\|\hat{\boldsymbol{\theta}}_S - \hat{\boldsymbol{\theta}}_{S'}\|_2$ after flipping a single label) across varying regularisation strengths $C$.
  The dashed line represents the \textbf{theoretical linear bound} ($\beta \propto C$) derived in Regime 3.
  Crucially, the empirical points follow this linear trend in the high-regularisation regime (small $C$) but naturally deviate and \textbf{saturate} around $C \approx 1.5$ as the logistic gradients vanish. This confirms that the model adheres to the stability bound while being strictly \textit{more} robust than predicted in low-regularisation settings.}
  \label{fig:stability_proof}
\end{figure}

\subsubsection{The Effective Dimension Bound (Refined Capacity)}
\label{sec:effective_dim}

The combinatorial VC bound ($D_k$) overestimates capacity because the Shapley basis functions are structurally correlated. We propose a tighter bound based on the \textbf{Effective Dimension} ($d_{\text{eff}}$).

\paragraph{Definition}
Let $\mathbf{\Sigma} = \mathbb{E}[\Phi(\mathbf{x})\Phi(\mathbf{x})^\top]$ be the covariance matrix of the feature mapping $\Phi(\mathbf{x})$. The effective dimension, or stable rank, is defined as:
\begin{equation*}
  d_{\text{eff}} = \frac{\left( \text{Tr}(\mathbf{\Sigma}) \right)^2}{\text{Tr}(\mathbf{\Sigma}^2)} = \frac{\left( \sum_{i=1}^{D_k} \lambda_i \right)^2}{\sum_{i=1}^{D_k} \lambda_i^2}
\end{equation*}
where $\lambda_i$ are the eigenvalues of $\mathbf{\Sigma}$.

\paragraph{Result}
For kernel-based methods (including Ridge regression on the Shapley basis), the refined bound replaces $D_k$ with $d_{\text{eff}}$:
\begin{equation*}
  R(h) \le \hat{R}_S(h) + \mathcal{O}\left( \sqrt{\frac{d_{\text{eff}}}{N}} \right).
\end{equation*}

\begin{figure}[t]
  \centering
  \includegraphics[width=0.75\linewidth]{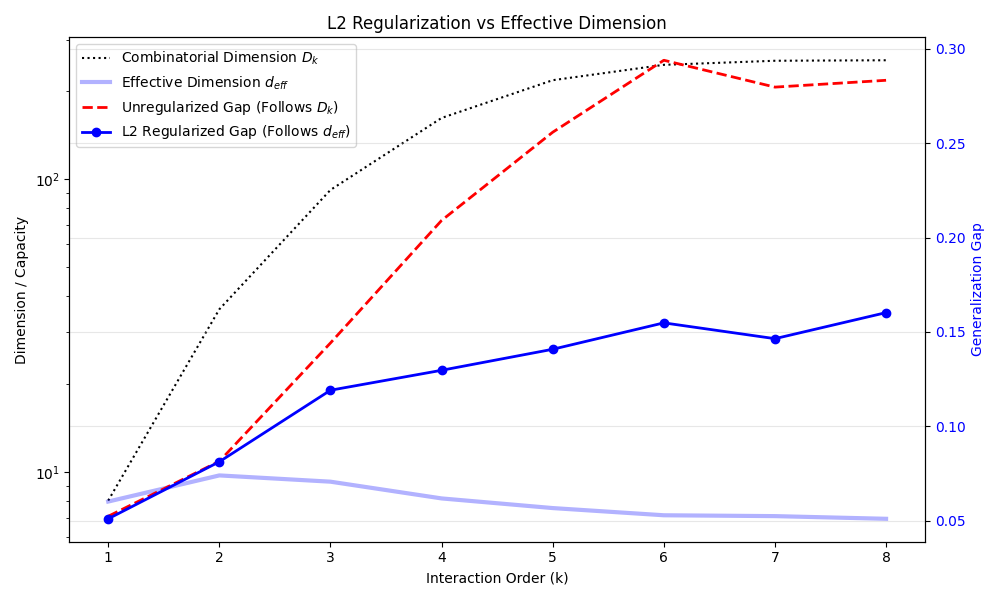}
  \caption{\textbf{Effective Dimension Validation.}
  Comparison of dimensions and generalization gaps (Settings: $N=1000, n=8$, averaged over 10 iterations).
  The plot demonstrates the divergence between the combinatorial dimension $D_k$ (black dotted line), which grows exponentially, and the effective dimension $d_{\text{eff}}$ (blue line), which saturates.
  Crucially, the generalization gap of the $\ell_2$-regularized model (blue dots) tracks the effective dimension, remaining stable even as $k$ increases, whereas the unregularized model (red dashed line) follows the exploding combinatorial complexity.}
  \label{fig:effective_dimension_proof}
\end{figure}

\subsubsection{Stability of Game-Theoretic Explanations}

While the previous sections established bounds on the predictive error, we now assess the robustness of our model's \textit{explanations}. An advantage of our architecture is that it is parameterised directly in the Shapley basis (Eq.~\ref{eq:shapley-agg}). Indeed, unlike models that require post-hoc attribution methods, which often introduce approximation noise, the coefficients of our model are the Shapley Interaction Indices.

We now formalise the connection between the regularisation used during training and the reliability of these game-theoretic explanations.

\begin{proposition}[Interaction Index Stability]
Let $\hat{\boldsymbol{\theta}}$ be the estimated parameter tuple, and $\boldsymbol{\theta}^*$ be the true tuple of the underlying game. In the Shapley-basis formulation, where the parameter $\theta_A$, $A\subseteq [n]$, is the Shapley Interaction Index $I(A)$, the $\ell_2$ regularisation provides a direct bound on the explanation variance. More precisely, if the parameter estimation error is bounded by $\varepsilon$ in the Euclidean norm, then the error in the interaction index for any coalition $A$ is bounded by $\varepsilon$.
\end{proposition}

\begin{proof}
We leverage the fundamental property of the Shapley basis parameterisation \cite{Grabisch1997KAdditive}: there is an identity mapping between the regression coefficient $\theta_A$ and the Shapley Interaction Index $I(A)$ for any coalition $A$. Therefore, an error in the explanation is exactly an error in the corresponding coefficient:
\begin{equation*}
|\hat{I}(A) - I^*(A)| = |\hat{\theta}_A - \theta^*_A|.
\end{equation*}
We now relate this individual error to the total model error. Let $\mathbf{e} = \hat{\boldsymbol{\theta}} - \boldsymbol{\theta}^*$ be the vector of estimated errors. The magnitude of any single component $|e_i|$ is strictly bounded by the Euclidean norm ($\ell_2$-norm) of the vector, i.e., $|e_i| \le \|\mathbf{e}\|_2$.
Since $\ell_2$ regularisation explicitly constraints the global parameter norm $\|\hat{\boldsymbol{\theta}} - \boldsymbol{\theta}^*\|_2 \le \epsilon$, it follows immediately that:
\begin{equation*}
|\hat{I}(A) - I^*(A)| \le \|\hat{\boldsymbol{\theta}} - \boldsymbol{\theta}^*\|_2 \le \epsilon.
\end{equation*}
This shows that constraining the global model complexity via Ridge regression imposes a hard ceiling on the instability of every individual interaction explanation.
\end{proof}

This result establishes a theoretical isomorphism between the model's algorithmic stability and its interpretability. While standard stability bounds (Regime 3) only guarantee robust \textit{predictions}, Proposition 1 guarantees robust \textit{explanations}.
Because $\ell_2$ regularisation explicitly minimizes the norm $\|\hat{\boldsymbol{\theta}}\|_2$, it imposes a hard ceiling on the variance of every Shapley interaction index as illustrated in Table~\ref{tab:summary_shapley_l2} in Supplementary Material (Appendix~\ref{app:performance_summary}).
As empirically demonstrated in Fig.~\ref{fig:stability_proof}, tightening the regularisation (lower $C$) linearly reduces the sensitivity $\beta$, which directly translates to trustworthy, noise-invariant game-theoretic explanations.

\section{Experimental Design: APDS Cohort}\label{sec:experiments}


In this section, we detail the data sourcing, phenotype extraction, and classification framework for the detection of APDS.

\subsection{Data Source and Cohort Selection}
We use French EHRs from the Necker hospital data warehouse (Dr. Warehouse) \cite{garcelon_clinician_2018}, where we identified a cohort of 222 patients (29 APDS and 193 Controls) from the Immunology department. Each patient had at least one hospital admission between 2018 and 2023. Our cohort is detailed in Supplementary Material (Appendix~\ref{app:data_apds_study}).

\subsection{Preprocessing}
We applied interquartile range (IQR) filtering to the number of clinical documents on the full dataset, excluding patients with excessively many or too few notes. Based on the IQR-derived lower and upper bounds, we retained only the first 20 notes in chronological order for each patient, and excluded patients with fewer than three notes to ensure sufficient clinical history for classification.



\subsection{Automated Phenotyping}
As structured phenotype data was unavailable, automated phenotyping of clinical notes was performed as described by \cite{otero_using_2022}. This approach combines a dictionary-based search for phenotypic terms with context-aware detection using deep learning. The dictionary relies on a subset of terms from the UMLS 2023AB vocabulary, focusing on semantic types related to phenotypes. The deep learning modules further enhance the extraction by performing disambiguation, subject attribution (assigning mentions to the patient versus family members), and negation detection (filtering out mentions of negated phenotypes).   
This process yielded a high-dimensional feature space of 2,563 unique phenotypes. Consequently, each patient $i$ is represented by a binary vector $\mathbf{x}_i \in \{0,1\}^{2563}$, where $x_{i,j}=1$ indicates the presence of phenotype $j$ in the patient's history. No dimensionality reduction was applied to preserve the interpretability of medical concepts.



\subsection{Classification Framework}
We formulated the task as a binary classification problem to distinguish APDS cases ($y=1$) from controls ($y=0$).

\paragraph{Baselines.}
We compared the proposed 2-additive SR against standard clinical scoring baselines and state-of-the-art ML models: $(i)$ classical \textbf{Logistic Regression (LR)} ~\cite{hosmer2013}, $(ii)$ tree-based ensembles suitable for tabular data representing few instances over many feautures, namely \textbf{Random Forest (RF)}~\cite{breiman2001} and \textbf{XGBoost} \cite{chen2016}, and $(iii)$ \textbf{feed-foward neural network (NN)} \cite{goodfellow2016}.


\section{Results and Discussion}


This section presents the empirical evaluation of SR, covering predictive performance, interpretability, and interaction. 


\subsection{Predictive Performance}
Table~\ref{tab:apds_performance} reports the classification performance across the nested cross-validation. The 2-additive SR model achieved the highest performance across several metrics, specifically reaching a {Balanced Accuracy of 0.941}, outperforming the linear baseline (0.907) and the non-linear XGBoost (0.884). Crucially for rare disease screening, the Shapley model demonstrated the highest {Sensitivity} (over 0.933), meaning it successfully identified the highest proportion of actual APDS patients. While Random Forest achieved higher specificity (fewer false positives), it did so at the cost of missing significant true cases (Sensitivity of 0.613). XGBoost demonstrated intermediate behavior, with higher variance across folds and lower balanced accuracy compared to the SR model. 

Despite comparable ROC-AUC values across models, differences in PR-AUC and F1-score highlight the importance of metrics that explicitly account for class imbalance in this setting. These results indicate that models achieving similar global discrimination may nonetheless differ substantially in their ability to detect rare cases. By optimizing for F1-score on the imbalanced cohort, our method strikes a clinically safer balance, prioritizing the detection of rare cases while maintaining high specificity (0.948).

Finally, we address the practical constraints of clinical deployment. We conducted thorough empirical study on several benchmark datasets with different characterics. To illustrate, we present in Table~\ref{tab:complexity} in the Supplementary Material (Appendix C), the results obtained on the Pima Indians Diabetes dataset\footnote{\url{https://www.kaggle.com/datasets/uciml/pima-indians-diabetes-database}} showing that the proposed model is computationally lightweight in both training and inference setting. These experiments were conducted on a standard CPU-based system, supporting the feasibility of widespread deployment in clinical settings without specialized hardware. Details of the CPU configuration used for the APDS experiments are provided in Appendix~\ref{app:data_apds_study}.



\begin{table*}[t]
\centering
\caption{Detailed clinical performance metrics across models using $\ell_2$ regularization. Results are reported as mean $\pm$ standard deviation over 5-nested cross-validation folds. Best-performing results are highlighted in bold. Underlined values indicate competing methods with comparable performance when standard deviation intervals overlap.}
\label{tab:apds_performance}
\resizebox{.9\textwidth}{!}{%
\begin{tabular}{lccccccc}
\toprule
\textbf{Model} & \textbf{Balanced Acc.} & \textbf{Sensitivity (Recall)} & \textbf{Specificity} & \textbf{Precision} & \textbf{F1-Score} & \textbf{ROC-AUC} & \textbf{PR-AUC}\\
\midrule
LR & \underline{$0.907 \pm 0.063$} & \underline{$0.860 \pm 0.127$} & \underline{$0.953 \pm 0.050$} & $0.789 \pm 0.212$ & \underline{$0.798 \pm 0.123$} &  \underline{$0.985 \pm 0.013$} & \underline{$0.941 \pm 0.044$} \\
RF & $0.804 \pm 0.052$ & $0.613 \pm 0.113$ & $\mathbf{0.995 \pm 0.010}$ & $\mathbf{0.960 \pm 0.080}$ & $0.737 \pm 0.060$ & \underline{$0.962 \pm 0.043$} & \underline{$0.893 \pm 0.071$}\\
XGBoost & $0.884 \pm 0.081$ & $0.793 \pm 0.163$ & $0.974 \pm 0.001$ & $0.810 \pm 0.033$ & \underline{$0.795 \pm 0.104$} & $0.914 \pm 0.061$ & $0.853 \pm 0.086$ \\
NN & \underline{$0.916 \pm 0.057$} & \underline{$0.893\pm 0.088$} & \underline{$0.938 \pm 0.053$} & $0.723 \pm 0.198$ & \underline{$0.784 \pm 0.138$} & $\mathbf{0.988 \pm 0.012}$ & $\mathbf{0.952 \pm 0.035}$ \\
\textbf{SR (k=2)} & $\mathbf{0.941 \pm 0.064}$ & $\mathbf{0.933 \pm 0.133}$ & \underline{$0.948 \pm 0.046$} & $0.760 \pm 0.168$ & $\mathbf{0.821 \pm 0.121}$ & \underline{$0.982 \pm 0.019$} & \underline{$0.936 \pm 0.062$} \\
\bottomrule
\end{tabular}%
}
\end{table*}

\subsection{Shapley Regression Interaction Analysis}
To analyze relationships among clinical phenotypes, we leveraged the game-theoretic structure of the SR model, which explicitly decomposes predictions into main effects and pairwise interaction terms. 

Firstly, we extracted Shapley-based coefficients from each fold of the nested cross-validation to identify phenotypes contributing to positive APDS predictions. The top 20 phenotypes, ranked by mean coefficients across folds, were subsequently reviewed by a clinical expert to distinguish known APDS manifestations (Table~\ref{tab:clinical_features_combined}, Supplementary Material, Appendix~\ref{app:apds_study}).
Secondly, we examined pairwise interactions learned by the SR model. For each trained model, the interaction coefficients were arranged into a symmetric interaction matrix, where entry $(i,j)$ represents the learned interaction strength between phenotypes $i$ and $j$. These interaction matrices were then aggregated across cross-validation folds by computing the mean interaction strength for each phenotype pair, yielding a consensus interaction matrix that reflects stable interaction patterns across training splits.

We focused the analysis on the most influential phenotypes, we computed a global interaction strength score for each phenotype by summing the absolute values of its interactions with all other phenotypes. The top 30 phenotypes with the highest total interaction strength were selected, and the interaction matrix was restricted to this subset.
To ensure robustness, we further assessed the stability of each interaction by calculating the fraction of folds in which the corresponding interaction coefficient was non-zero. Interactions observed in fewer than 70\% of folds were considered unstable and excluded from visualization. For the selected set of phenotypes, we examined both individual main effects, illustrated as a bar plot ( Figure~\ref{fig:apds_main_effects})
, and pairwise phenotype interactions, represented as an interaction matrix (Figure~\ref{fig:apds_coeff}). The filtered interaction matrix was visualized as a heatmap, with color intensity indicating the mean interaction strength and sign across folds. Both the main effects plot and the interaction heatmap were jointly reviewed with a clinical expert to evaluate the plausibility of the interaction patterns identified by the SR model. Among the 30 phenotypes analyzed, 17 were identified as core clinical manifestations associated with APDS.


\begin{figure}[t]
  \centering
  \includegraphics[width=0.9\linewidth]{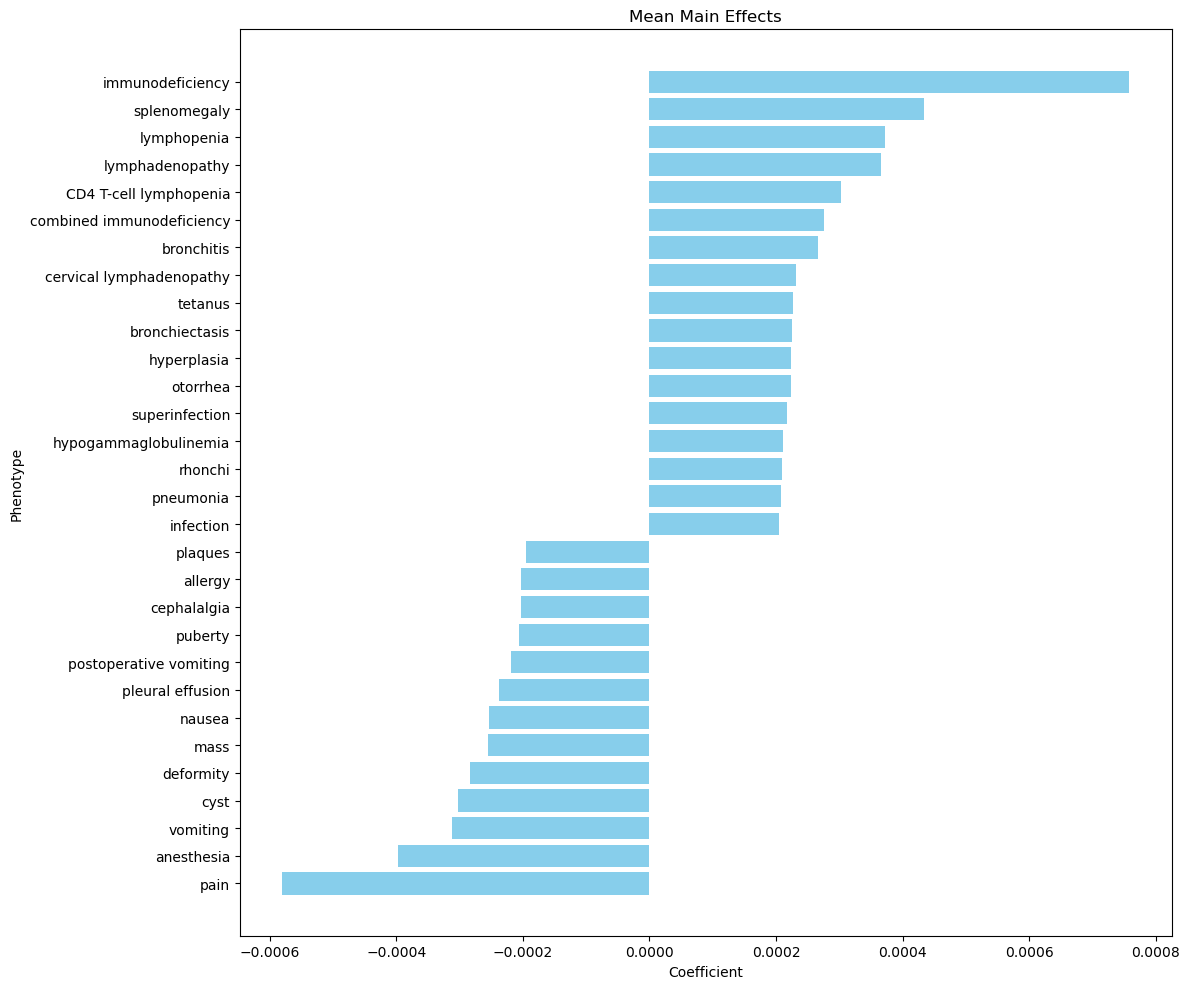}
  \caption{Marginal contribution (main effects) among phenotypes in APDS dataset.} 
  \label{fig:apds_main_effects}
\end{figure}

\begin{figure}[t]
  \centering
  \includegraphics[width=0.9\linewidth]{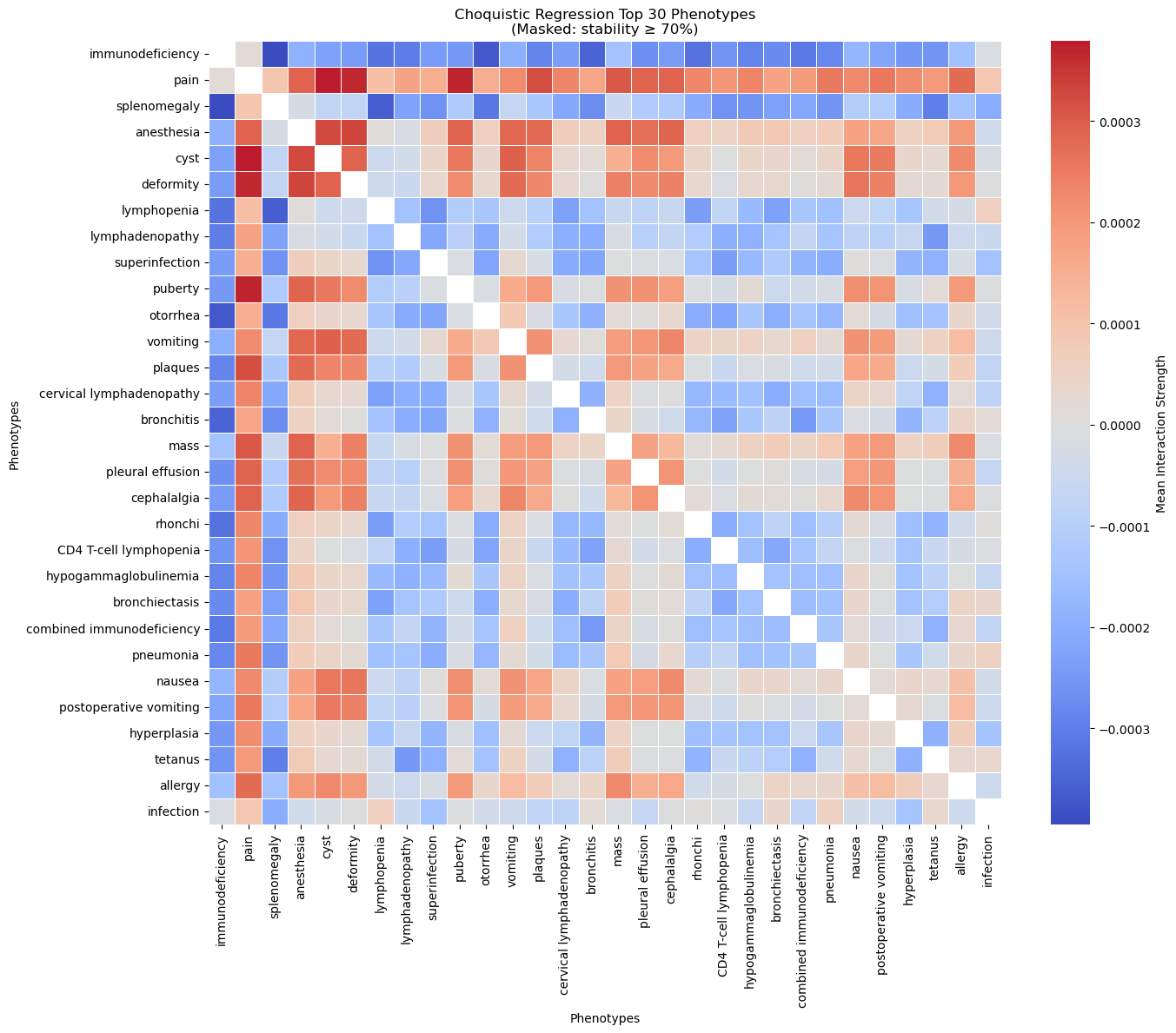}
  \caption{\textbf{Interaction effects among phenotypes in APDS dataset.} The heatmap illustrates the mean pairwise interaction strength between the top interacting phenotypes. Red indicates positive interaction (synergy), while blue indicates negative (redundancy).
  }
  \label{fig:apds_coeff}
\end{figure}

As shown in Figure~\ref{fig:apds_main_effects}, \textit{tetanus} and \textit{rhonchi} emerged as positive predictors of APDS. This finding aligns with previous analyses discussed in Supplementary Material (Appendix~\ref{app:exp_apds_study}), where \textit{tetanus} was attributed to vaccination records, and \textit{rhonchi} reflected respiratory signs associated with other symptoms. Interestingly, two phenotypes recognized by the clinical expert as APDS-associated—\textit{pleural effusion} and \textit{allergy}—appeared to contribute to classifying individuals as not having APDS. This counterintuitive pattern likely reflects overlapping phenotypes present in both APDS and control cohorts.

To further explore these relationships, the filtered interaction matrix was visualized as a heatmap (Figure~\ref{fig:apds_coeff}), illustrating the pairwise interactions among the selected phenotypes. In this heatmap, five core APDS phenotypes—\textit{immunodeficiency}, \textit{splenomegaly}, \textit{lymphopenia}, \textit{lymphadenopathy}, and \textit{superinfection}—displayed predominantly negative interactions with many other phenotypes. Within the SR framework, such negative interactions reflect redundancy rather than antagonism: these features are already highly informative individually, and their co-occurrence with additional phenotypes contributes limited incremental information to the model. This behavior is consistent with clinical manifestations commonly reported for APDS patients in Orphanet \cite{pavan_clinical_2017}, indicating that the model appropriately captures their central role in disease characterization.
Other APDS-associated phenotypes, including \textit{otorrhea}, \textit{cervical lymphadenopathy}, \textit{bronchitis}, \textit{CD4 T-cell lymphopenia}, \textit{hypogammaglobulinemia}, \textit{bronchiectasis}, \textit{combined immunodeficiency}, \textit{hyperplasia}, \textit{pneumonia}, \textit{allergy}, and \textit{infection}, exhibited a more heterogeneous pattern of positive, negative, and neutral interactions. Notably, several of these phenotypes showed negative interactions with core APDS features, again suggesting partial redundancy with already informative immunological and lymphoproliferative manifestations. This variability indicates that their contribution to APDS characterization might be context-dependent and modulated by the presence of other clinical findings. 

In contrast, although \textit{pleural effusion} and \textit{allergy} appeared negatively associated with APDS in isolation, they demonstrated strong positive interactions with phenotypes not traditionally linked to APDS. For example, \textit{pleural effusion} showed positive associations with \textit{pain}, \textit{anesthesia}, and \textit{deformity}, among others. The interaction matrix also highlights the role of phenotypes not classically associated with APDS, including \textit{pain}, \textit{anesthesia}, \textit{cyst}, \textit{deformity}, and \textit{cephalalgia}, which tended to interact positively with multiple other phenotypes, both APDS-associated and non-associated. For instance, positive interactions were observed among non-APDS phenotypes themselves, such as \textit{nausea} and \textit{cephalalgia}. These patterns suggest that such phenotypes may reflect downstream clinical consequences, comorbidities, or procedural contexts rather than primary drivers of the disease. Interestingly, \textit{pain} exhibited widespread positive interactions with both APDS and non-APDS phenotypes, despite not emerging as a strong individual predictor of APDS. This observation is consistent with clinical expectations, as pain is a non-specific symptom whose diagnostic significance depends largely on its context and association with other clinical manifestations rather than its isolated occurrence.

Overall, the interaction analysis highlights a clear distinction between core APDS phenotypes that act as strong standalone predictors and secondary or contextual phenotypes whose importance emerges primarily through interactions. This demonstrates the utility of the game-theoretic framework in disentangling redundant, complementary, and context-dependent clinical signals, providing an interpretable view of symptom co-occurrence patterns in APDS.

\section{Conclusion}

We presented a framework for interpretable rare disease diagnostic support, evaluated on APDS. By employing a 2-additive Shapley model with robust regularization, we effectively addressed the high dimensionality and data scarcity inherent to this domain. Our results show that Shapley Regression is competitive with state-of-the-art models such as XGBoost in predictive accuracy. In our use case, it also improves sensitivity, albeit with a modest decrease in specificity, while providing a key advantage: explanatory global Shapley Interaction Indices that offer insights into phenotype interactions.



For APDS, this framework allows us to confirm known APDS-associated phenotypes, as reported in Orphanet \cite{pavan_clinical_2017} and recognized by clinicians (see Table~\ref{tab:clinical_features_combined} in Supplementary Material, Appenidx~\ref{app:exp_apds_study}). More importantly, it enables systematic analysis of how phenotypes jointly contribute through interactions rather than in isolation. This interaction-focused analysis provides a complementary perspective on symptom co-occurrence patterns not captured by traditional models and has the potential to improve our understanding of disease mechanisms and clinical manifestations.
While the current cohort size limits the strength of predictive conclusions, increasing the number of APDS patients would allow for more robust validation of the identified interactions, improved predictive performance, and greater confidence in the stability and clinical relevance of the results. These findings suggest that game-theoretic models such as SR may serve as useful exploratory tools for understanding complex phenotypic interactions in rare diseases. 

As future work, we plan to further investigate theoretical bounds (Sec.~\ref{sec:theoretical-bounds}), considering dataset characteristics and capacity distributions. Also, we saw that SR outperforms the other models considered in terms of time complexity, when the number of features is reasonable (see Table \ref{tab:complexity} in Supplementary Material, Appendix~\ref{app:complexity}). However, sampling techniques should be developed for high number of features. Finally, negative interactions, which are of particular interest for medical applications like APDS detection, will be explored in collaboration with our clinical team.


\section*{Ethical Statement}

We propose a lightweight approach that achieves performance competitive while remaining computationally efficient. Notably, the proposed method can be easily implemented on standard personal computers using only CPU resources.

This research was approved by the Necker Hospital Data Privacy Officer under the registration number \#2026 04171640. The study relied exclusively on collected clinical data that were fully de-identified prior to analysis, and all computational findings were reviewed in conjunction with expert clinical interpretation. A language model (ChatGPT) was used to assist in editing text for clarity and readability. 


\section*{Acknowledgments}
This work was partially supported by the French National Research Agency (ANR)\footnote{Agence Nationale de la Recherche.} project ``Analogies: from theory to tools and applications'' (AT2TA), ANR-22-CE23-0023, and partially funded under the ``Investissements d’Avenir'' program, ANR-10-IAHU-01.
This work was also supported by Portuguese national funds through FCT under project UIDB/50021/2020 (DOI: 10.54499/UIDB/50021/2020).
Additionally, the work was partially supported by the project Center for Responsible AI, reference C628696807-00454142. 


\newpage
\bibliographystyle{named}
\bibliography{ijcai26}

\clearpage
\appendix

\section{Mathematical Foundations}\label{sec:math-foundations}
In the main text, we presented Shapley Regression directly in terms of interaction indices to facilitate clinical interpretation. Here, we provide the underlying measure-theoretic definitions that justify this decomposition.

\subsection{Capacities and Möbius Transforms}
The proposed model is grounded in the theory of capacities (fuzzy measures).
\begin{definition}[Capacity]
A capacity on a set of features $\mathcal{F} = \{1,\ldots,n\}$ is a set function $\mu: 2^\mathcal{F} \rightarrow [0,1]$ satisfying boundary conditions $\mu(\emptyset)=0, \mu(\mathcal{F})=1$ and monotonicity ($A \subseteq B \Rightarrow \mu(A) \leq \mu(B)$).
\end{definition}

The \textbf{Möbius transform} $m_\mu: 2^\mathcal{F} \rightarrow \mathbb{R}$ provides a canonical representation of the capacity, defined as:
\begin{equation*}
m_\mu(A) = \sum_{B \subseteq A} (-1)^{|A \setminus B|} \mu(B), \quad \forall A \subseteq \mathcal{F}
\end{equation*}

The Choquet integral with respect to $\mu$ can be computed using the Möbius coefficients:
\[
C_\mu(x)=\sum_{T\subseteq\mathcal{F}}m(T)\,\min_{i\in T}x_i
\]
The $k$-additive assumption used in the main paper corresponds to setting $m(A) = 0$ for all $|A| > k$.

\subsection{Relationship to Shapley Indices}
The Shapley Interaction Indices $I(A)$ used in our regression formulation are linearly related to the Möbius coefficients. They can be recovered via the transformation:
\begin{equation*}
I(A) = \sum_{B \subseteq \mathcal{F} \setminus A} \frac{(n-|B|-|A|)!|B|!}{(n-|A|+1)!} m(A \cup B)
\end{equation*}
This linear bijection allows us to optimize the model directly in the Shapley basis, as performed in Eq.~\ref{eq:shapley-agg}, ensuring that the learned weights naturally correspond to game-theoretic interactions.

\subsection{Interpretability Analysis} \label{interpretability-analysis}

Recall from Section~\ref{sec:model-training} (Eq.~\eqref{eq:shapley-agg}) that, once the
game values are expressed in the Shapley basis, each logistic coefficient
\(\phi_{A}\) equals the Shapley interaction index \(I(A)\). Plotting
these coefficients therefore reveals directly which terms up to size $k$ the model deems most influential.

\paragraph{Gradient Consistency}
A unique theoretical advantage of the Game-based formulation is the exact consistency between global and local explanations. For the Lovász extension of any game $v$, Grabisch \cite{Grabisch2016SetFunctions} proves that the Shapley value $\phi_i(v)$ is exactly the expected partial derivative of the function over the domain $[0,1]^n$:
\begin{equation*}
  \theta_{\{i\}}^{\text{Shapley}} = \int_{[0,1]^n} \frac{\partial \text{logit}(x)}{\partial x_i} \, dx
\end{equation*}
This property guarantees that our learned coefficients are not merely heuristic attribution scores, but faithful representations of the model's average sensitivity to each feature. By learning $\boldsymbol{\theta}$ directly in the Shapley basis, SR aligns the optimization objective (minimizing loss) with the explanation objective (accurate sensitivity attribution).

\newpage

\section{Proofs of Generalization Bounds}\label{app:proofs}

In Section~\ref{sec:theoretical-bounds}, we presented generalisation bounds for different complexity regimes. Here, we provide formal derivations for the $\ell_2$-regularised case (Algorithmic Stability) and the $\ell_1$-regularised case (Rademacher Complexity).

\subsection{Algorithmic Stability (Regime 3)}
We rely on the framework of \textit{Uniform Stability} introduced by Bousquet and Elisseeff.

\begin{definition}[Uniform Stability]
A learning algorithm $\mathcal{A}$ has uniform stability $\beta$ if, for any dataset $S$ of size $N$ and any modification $S^{\prime}$ obtained by replacing a single example in $S$, the change in the loss of the learned hypothesis $h_S$ is bounded by:
\[
\sup_{z} | \ell(h_S, z) - \ell(h_{S^{\prime}}, z) | \leq \beta
\]
\end{definition}

\begin{lemma}[Stability of Regularised Loss]
Let the loss function $\ell(\theta, z)$ be convex and $L$-Lipschitz with respect to $\theta$. If we minimize a regularized objective $J(\theta) = \frac{1}{N}\sum \ell(\theta, z_i) + \frac{\lambda}{2}\|\theta\|_2^2$, then the algorithm has stability:
\[
\beta \leq \frac{L^2}{\lambda N}
\]
\end{lemma}

\begin{proof}[Proof Sketch for Shapley Regression]
Our loss function (Logistic Regression on the Shapley basis) is:
\[
\mathcal{L}(\theta) = \log(1 + \exp(-y \langle \theta, \Phi(x) \rangle))
\]
The gradient with respect to $\theta$ is bounded by the norm of the feature vector:
\[
\|\nabla_\theta \mathcal{L}\|_2 \le \sup_{x} \|\Phi(x)\|_2
\]
For a $k$-additive Choquet integral ($k=2$), the feature mapping $\Phi(x)$ consists of $D_k$ basis functions $\phi_A(x)$. Since inputs $x \in [0,1]^n$, the basis functions are bounded. Let $B_\Phi = \sup \|\Phi(x)\|_2$. Then the loss is $B_\Phi$-Lipschitz.
Applying the Lemma yields the bound stated in Eq.~(6):
\[
R(h) - \hat{R}_N(h) \le \frac{2 B_\Phi^2}{\lambda N}
\]
This confirms that increasing $\lambda$ directly reduces the generalization gap, compensating for the increase in dimensionality $B_\Phi^2$.
\end{proof}

\subsection{Rademacher Complexity (Regime 2)}
For $\ell_1$ regularisation, we restrict $\mathcal{H}$ to linear predictors with $\|\theta\|_1 \le B$.

\begin{theorem}[Rademacher Bound for $\ell_1$ Constraints]
For linear predictors $f(x) = \langle \theta, x \rangle$ with $\|\theta\|_1 \le B$ and $\|x\|_\infty \le X_\infty$, the empirical Rademacher complexity is bounded by:
\[
\hat{\mathfrak{R}}_S(\mathcal{H}) \le B X_\infty \sqrt{\frac{2 \ln(2n)}{N}}
\]
\end{theorem}

\begin{proof}[Application to Shapley Regression]
In our context, the ``features'' are the $D_k$ basis functions.
1. The inputs are bounded: $|\phi_A(x)| \le 1$ for all $A$, so $X_\infty = 1$.
2. The $\ell_1$-regularized Shapley Regression constrains $\sum |I(A)| \le B$.
3. Substituting these into the standard Rademacher bound yields:
\[
R(h) - \hat{R}_N(h) \le 2 B \sqrt{\frac{2 \ln(2 D_k)}{N}}
\]
This derivation confirms Eq.~(5), showing the logarithmic dependence on the dimension $D_k$.
\end{proof}




\newpage

\section{Computational Complexity Analysis}\label{app:complexity}

While Shapley Regression significantly outperforms the linear baseline in predictive power, it introduces additional computational cost due to the expansion of the feature space. In a $k$-additive model ($k=2$), the number of features increases from $n$ to $n + \binom{n}{2}$, reflecting the explicit modeling of pairwise interactions.

All experiments were conducted on a standard commodity laptop (CPU only, no GPU acceleration). Table~\ref{tab:complexity} reports the computational resources required by the models evaluated in this study on the Pima Indians Diabetes dataset, including training time, inference time, model size, and estimated floating-point operations (FLOPs), averaged over the outer folds of a 5-fold nested cross-validation.

\begin{table}[h]
\centering
\caption{Computational efficiency and model complexity across methods on the Pima Indians Diabetes dataset. Reported values correspond to the mean training time, inference time, model size, and estimated floating-point operations (FLOPs), averaged over the outer folds of a 5-fold nested cross-validation.}
\label{tab:complexity}
\resizebox{\columnwidth}{!}{%
\begin{tabular}{lcccc}
\toprule
\textbf{Model} & \textbf{Training Time (s)} & \textbf{Inference Time (s)} & \textbf{Model Size (MB)} & \textbf{FLOPs} \\
\midrule
Logistic Regression & 0.0021 & 0.0005 & 0.0010 & 2457.6\\
Random Forest & 0.2393 & 0.0343 & 0.4192 &181240.2 \\
XGBoost & 0.0861 & 0.0025 & 0.3758 & 159640.0\\
NN & 4.271 & 0.0010 & 0.0076 & 63086.4 \\
\textbf{SR (k=2)} & 0.0020 & 0.0005 & 0.0020 & 9830.4\\
\bottomrule
\end{tabular}%
}
\end{table}

\paragraph{The Cost of Interpretability}

As shown in Table~\ref{tab:complexity}, Shapley Regression with $k=2$ incurs only a modest computational overhead compared to standard logistic regression, despite the expanded feature space. In practice, its training and inference times remain comparable to linear models and substantially lower than those of tree-based and neural network approaches, while the model size increases only marginally.

In the context of rare disease diagnosis, this highlights the model's:
\begin{enumerate}
  \item \textbf{Feasibility:} The absolute computational cost remains low. A model that trains and runs efficiently on a standard CPU without specialized hardware is well suited for clinical research and deployment, in contrast to deep learning approaches that often require GPU acceleration and substantially higher computational resources.
  \item \textbf{Value Proposition:} The additional computation directly "purchases" two critical properties: \textbf{convexity} (guaranteeing a unique, stable solution) and \textbf{exact game-theoretic interpretability}. Unlike black-box neural networks where the cost goes into opaque hidden layers, here the cost goes into explicitly calculating every pairwise symptom interaction.
\end{enumerate}

Thus, SR occupies a ``sweet spot'': it is powerful enough to capture non-linearities (unlike LR), yet simple enough to run locally without specialized hardware (unlike DL), making it an ideal candidate for decentralized deployment in hospitals.

\newpage

\section{Technical Validation Details}\label{app:technical_validation}

To justify the use of the 2-additive SR model for the medical use case, we performed extensive benchmarking on synthetic and public datasets. This section details the implementation and protocols for those experiments.

\subsection{Implementation}
Our modular Python framework consists of:
\begin{itemize}
  \item \textbf{Representation Modules:} Handling conversions between Möbius, Shapley, and game-based parameterizations.
  \item \textbf{Logistic Integration:} A custom class wrapping the representation modules to compute forward passes, gradients, and cross-entropy loss within a standard \texttt{sklearn}-compatible interface.
  \item \textbf{Codebase:} The full implementation is available at the github repository\footnote{\url{https://github.com/tomasbrogueira/shapley_regression}}.
\end{itemize}

\subsection{Benchmark Datasets}
We validated the method on eight public binary classification datasets from the UCI repository and other sources, covering diverse domains (medical, financial, agricultural).
\begin{itemize}
  \item \textbf{Synthetic Data:} To test interaction recovery, we generated a dataset (\texttt{pure\_pairwise}) where the ground truth log-odds are determined strictly by multiplicative interactions ($x_i \cdot x_j$).
  \item \textbf{Public Datasets:} Statistics for the eight real-world datasets (including Diabetes, COVID-19, and Mammographic screening) are provided in Table \ref{tab:benchmark_datasets}. All continuous attributes were rescaled to $[0,1]$.
\end{itemize}

\begin{table}[h]
\caption{Descriptive statistics and sources of public, private, and synthetic datasets. The column ``Benchmarking'' indicates datasets included in the benchmark experiments.}
\centering
\label{tab:benchmark_datasets}
\resizebox{\columnwidth}{!}{%
\begin{tabular}{llccclc}
\toprule
\textbf{Dataset} & \textbf{Domain} & \textbf{Samples} & \textbf{Features} & \textbf{Class Balance} & \textbf{Source / Link} & \textbf{Benchmarking}\\
\midrule
Necker Hospital APDS Dataset & Medicine & 222 & 2,563 & 193 / 29 & Private & No \\
Dados Covid SBPO & Medicine & 64,174 & 9 & 63,033 / 1,141 & Private & Yes \\
Banknote Authentication & Finance & 1,372 & 4 & 762 / 610 & \href{https://archive.ics.uci.edu/dataset/267/banknote+authentication}{UCI Repository} & Yes \\
Blood Transfusion & Medicine & 748 & 4 & 570 / 178 & \href{https://www.kaggle.com/datasets/whenamancodes/blood-transfusion-dataset}{Kaggle} & Yes\\
Mammographic Mass & Medicine & 831 & 4 & 428 / 403 & \href{https://archive.ics.uci.edu/dataset/161/mammographic+mass}{UCI Repository} & Yes \\
Raisin & Agriculture & 900 & 7 & 450 / 450 & \href{https://archive.ics.uci.edu/dataset/850/raisin}{UCI Repository} & Yes \\
Rice (Cammeo/Osmancik) & Agriculture & 3,810 & 7 & 1,630 / 2,180 & \href{https://www.kaggle.com/datasets/muratkokludataset/rice-dataset-commeo-and-osmancik}{Kaggle} & Yes \\
Pima Indians Diabetes & Medicine & 768 & 8 & 500 / 268 & \href{https://www.kaggle.com/datasets/uciml/pima-indians-diabetes-database}{Kaggle} & Yes \\
Skin Segmentation & Medicine & 245,056 & 4 & 50,858 / 194,198 & \href{https://archive.ics.uci.edu/dataset/229/skin+segmentation}{UCI Repository} & Yes \\
Pure Pairwise & Synthetic & 1,000 & 15 & 500 / 500 & Generated & Yes \\
\bottomrule
\end{tabular}%
}
\end{table}

\subsection{Benchmark Protocol}
For the technical validation (results in Supplementary Material in Appendix F):
\begin{itemize}
  \item We performed a sweep of the additivity parameter $k=1,\dots,n$.
  \item We compared three regularization regimes: None, $\ell_1$ (Lasso), and $\ell_2$ (Ridge).
  \item \textbf{Noise Robustness:} Evaluated by adding zero-mean Gaussian noise ($\sigma \in \{0.1, 0.2, 0.3\}$) to test inputs.
  \item \textbf{Stability:} Computed via bootstrap accuracy on 50 resamples.
\end{itemize}
These experiments confirmed that the $\ell_2$-regularized 2-additive model offers the optimal trade-off for high-stakes domains, motivating its selection for the APDS study.

\newpage
\section{Overall Performance Summary} \label{app:performance_summary}

This section summarizes the performance of the SR model under different regularization settings (L1, L2, None), as illustrated in Table~\ref{tab:summary_shapley_l1}, Table~\ref{tab:summary_shapley_l2}, and Table~\ref{tab:summary_shapley_none}, respectively. For each dataset, we report the k-value that achieves the highest accuracy, the k-value that provides the best stability under bootstrapping, and the k-value that shows the most robustness to noise in the input data. Accuracy values are presented as mean $\pm$ standard deviation.

\begin{table*}[h]
\centering
\caption{Overall Performance Summary for Shapley None}
\label{tab:summary_shapley_none}
\resizebox{\textwidth}{!}{%
\begin{tabular}{lrlrlrr}
\toprule
Dataset & Best K (Acc) & Accuracy ($\pm$ Std) & Best K (Robust) & Robustness Accuracy ($\pm$ Std) & Best K (Stab) & Bootstrap Stability (Std) \\
\midrule
banknotes & 1 & $1.0000 \pm 0.0000$ & 1 & $1.0000 \pm 0.0000$ & 1 & 0.0000 \\
dados\_covid\_sbpo\_atual & 5 & $0.6951 \pm 0.0036$ & 3 & $0.6872 \pm 0.0007$ & 4 & 0.0025 \\
diabetes & 5 & $0.7933 \pm 0.0255$ & 4 & $0.7513 \pm 0.0243$ & 2 & 0.0231 \\
mammographic & 2 & $0.8171 \pm 0.0270$ & 2 & $0.8140 \pm 0.0067$ & 1 & 0.0242 \\
pure\_pairwise\_interaction & 2 & $0.7800 \pm 0.0251$ & 4 & $0.7440 \pm 0.0142$ & 10 & 0.0231 \\
raisin & 1 & $0.8593 \pm 0.0216$ & 1 & $0.7896 \pm 0.0116$ & 5 & 0.0195 \\
rice & 2 & $0.9266 \pm 0.0071$ & 7 & $0.8720 \pm 0.0076$ & 2 & 0.0071 \\
skin & 1 & $1.0000 \pm 0.0000$ & 1 & $0.9996 \pm 0.0001$ & 1 & 0.0000 \\
transfusion & 2 & $0.7281 \pm 0.0325$ & 3 & $0.7140 \pm 0.0107$ & 1 & 0.0230 \\
\bottomrule
\end{tabular}%
}
\end{table*}

\begin{table*}[h]
\centering
\caption{Overall Performance Summary for Shapley L1}
\label{tab:summary_shapley_l1}
\resizebox{\textwidth}{!}{%
\begin{tabular}{lrlrlrr}
\toprule
Dataset & Best K (Acc) & Accuracy ($\Pm$ Std) & Best K (Robust) & Robustness Accuracy ($\Pm$ Std) & Best K (Stab) & Bootstrap Stability (Std) \\
\midrule
banknotes & 1 & $1.0000 \pm 0.0000$ & 4 & $1.0000 \pm 0.0000$ & 1 & 0.0000 \\
dados\_covid\_sbpo\_atual & 6 & $0.6946 \pm 0.0034$ & 8 & $0.6904 \pm 0.0008$ & 1 & 0.0028 \\
diabetes & 2 & $0.7733 \pm 0.0311$ & 6 & $0.7653 \pm 0.0153$ & 4 & 0.0230 \\
mammographic & 2 & $0.8132 \pm 0.0261$ & 3 & $0.8062 \pm 0.0103$ & 4 & 0.0224 \\
pure\_pairwise\_interaction & 2 & $0.7400 \pm 0.0222$ & 7 & $0.7313 \pm 0.0120$ & 2 & 0.0222 \\
raisin & 2 & $0.8667 \pm 0.0207$ & 5 & $0.8467 \pm 0.0127$ & 5 & 0.0197 \\
rice & 1 & $0.9228 \pm 0.0091$ & 4 & $0.9072 \pm 0.0046$ & 6 & 0.0068 \\
skin & 1 & $1.0000 \pm 0.0000$ & 4 & $0.9996 \pm 0.0000$ & 1 & 0.0000 \\
transfusion & 1 & $0.7135 \pm 0.0323$ & 2 & $0.7070 \pm 0.0113$ & 3 & 0.0267 \\
\bottomrule
\end{tabular}%
}
\end{table*}

\begin{table*}[h]
\centering
\caption{Overall Performance Summary for Shapley L2}
\label{tab:summary_shapley_l2}
\resizebox{\textwidth}{!}{%
\begin{tabular}{lrlrlrr}
\toprule
Dataset & Best K (Acc) & Accuracy ($\pm$ Std) & Best K (Robust) & Robustness Accuracy ($\pm$ Std) & Best K (Stab) & Bootstrap Stability (Std) \\
\midrule
banknotes & 1 & $1.0000 \pm 0.0000$ & 1 & $1.0000 \pm 0.0000$ & 1 & 0.0000 \\
dados\_covid\_sbpo\_atual & 7 & $0.6951 \pm 0.0027$ & 6 & $0.6899 \pm 0.0008$ & 4 & 0.0025 \\
diabetes & 4 & $0.7667 \pm 0.0297$ & 5 & $0.7667 \pm 0.0180$ & 5 & 0.0241 \\
mammographic & 2 & $0.8093 \pm 0.0269$ & 2 & $0.8086 \pm 0.0057$ & 1 & 0.0264 \\
pure\_pairwise\_interaction & 4 & $0.7267 \pm 0.0273$ & 4 & $0.7100 \pm 0.0135$ & 11 & 0.0224 \\
raisin & 2 & $0.8704 \pm 0.0239$ & 7 & $0.8519 \pm 0.0110$ & 1 & 0.0211 \\
rice & 1 & $0.9251 \pm 0.0089$ & 5 & $0.9165 \pm 0.0039$ & 2 & 0.0075 \\
skin & 1 & $1.0000 \pm 0.0000$ & 1 & $0.9998 \pm 0.0000$ & 1 & 0.0000 \\
transfusion & 2 & $0.7105 \pm 0.0318$ & 3 & $0.7058 \pm 0.0129$ & 1 & 0.0230 \\
\bottomrule
\end{tabular}%
}
\end{table*}

\newpage

\section{Noise Robustness Analysis}

We also carried out an extensive robustness study taking into account different configurations. In particular, we provide a detailed breakdown of model accuracy under such as different regularizations, values of $k$ and levels noise in the github repository\footnote{\label{note:github}\url{https://github.com/tomasbrogueira/shapley_regression}}.

\section{APDS Study} \label{app:apds_study}

In this appendix we provide further experimental details and findings on the APDS cohort.
\subsection{Data Source and Cohort Selection} \label{app:data_apds_study}

Our cohort is composed of:
\begin{itemize}
  \item \textbf{APDS Group ($n=29$):} Included patients with a confirmed diagnosis of APDS, determined via specific keywords: “APDS,” “APDS1,” “APDS2,” “PIK3CD,” “PIK3R1,” “PI3K,” and “PI3KCD.” Diagnosis dates were determined by the date of the first document mentioning these keywords. All clinical notes dated strictly after the determined diagnosis date were removed from the dataset. 

  \item \textbf{Control Group ($n=193$):} Included control patients (without APDS) selected from the same department. Controls were selected to have clinical characteristics similar to the APDS cohort, including comparable numbers of clinical events and pathways.
\end{itemize}

For this study, we conducted extensive experiments comparing classical Logistic Regression (LR), tree-based ensembles (Random Forest and XGBoost), and a feed-forward neural network (NN).
\paragraph{Protocol.}
All models were trained using \textbf{nested 5-fold cross-validation}. The inner loop (3 folds) was used for hyperparameter tuning (e.g., regularization strength $\lambda$ for LR, tree depth for XGBoost), while the outer loop assessed performance. Due to the class imbalance (29 \textit{vs.} 193), we optimized for F1-score and report balanced accuracy, ROC-AUC, and Precision-Recall AUC.

Models were evaluated using different regularization schemes (L1, L2, or none) and strategies to handle class imbalance (class weighting alone or combined with undersampling at a 0.33 ratio). For each experiment, we measured training time, inference time, model size, and estimated floating-point operations (FLOPs). All APDS experiments were executed on a Linux server with an AMD Ryzen 7 5800X CPU (8 cores / 16 threads, 3.8~GHz) and 62~GB of RAM. Detailed scripts, hyperparameters, and results are available in the github repository\footnotemark[\getrefnumber{note:github}].

\subsection{Main Findings} \label{app:exp_apds_study}

In Table~\ref{tab:clinical_features_combined} 
we report the top 20 positively contributing features computed by their mean Shapley values derived from the Shapley Regression model. An expert-driven clinical review of these individual phenotypes indicated that 18 of the 20 features are directly consistent with known APDS manifestations and overlap with clinical manifestations reported for APDS in reference resources such as Orphanet \cite{pavan_clinical_2017}. Two features, \textit{tetanus} and \textit{rhonchi}, were not immediately associated with APDS classification. Upon further clinical interpretation, the presence of \textit{tetanus} was exclusively attributed to the documentation of tetanus vaccination status rather than to the observation of an active disease, explaining its appearance among extracted phenotypes. Similarly, \textit{rhonchi} was interpreted as a respiratory sign likely related to recurrent pulmonary infections, airway inflammation, or lymphoid hyperplasia—conditions frequently observed in APDS patients. Several individuals in the cohort exhibited such respiratory symptoms, supporting the relevance of this feature in the model's predictions.
Overall, this analysis demonstrates that the majority of highly ranked features align with established clinical characteristics of APDS, while the remaining features can be plausibly explained through contextual clinical interpretation, reinforcing the interpretability of the model.



\begin{table*}[t]
\centering
\caption{Top 20 positively contributing phenotypes for APDS classification identified by the Shapley Regression model, ranked by mean Shapley coefficients across 5-fold cross-validation. The column \emph{Clinical expert opinion} indicates whether a phenotype is considered clinically associated with APDS according to expert review, while \emph{Reported in Orphanet} indicates whether the phenotype is listed in Orphanet for APDS. A checkmark (\checkmark) denotes presence or association, and -- denotes absence or no documented association.}
\label{tab:clinical_features_combined}
\resizebox{\textwidth}{!}{%
\begin{tabular}{lcc}
\toprule
\textbf{Phenotype} & \textbf{Clinical expert opinion} & \textbf{Reported in Orphanet} \\
\midrule
Immunodeficiency & \checkmark & \checkmark \\
Splenomegaly & \checkmark & \checkmark \\
Lymphopenia & \checkmark & -- \\
Lymphadenopathy & \checkmark & \checkmark \\
CD4 T-cell lymphopenia & \checkmark & -- \\
Combined immunodeficiency & \checkmark & \checkmark \\
Bronchitis & \checkmark & -- \\
Cervical lymphadenopathy & \checkmark & \checkmark \\
Bronchiectasis & \checkmark & \checkmark \\
Hyperplasia & \checkmark & \checkmark \\
Otorrhea & \checkmark & \checkmark \\
Superinfection & \checkmark & -- \\
Hypogammaglobulinemia & \checkmark & \checkmark \\
Pneumonia & \checkmark & \checkmark \\
Infection & \checkmark & \checkmark \\
Rhinorrhea & \checkmark & -- \\
Bronchial infection & \checkmark & \checkmark \\
Lymphoma & \checkmark & \checkmark \\
Tetanus & -- & -- \\
Rhonchi & -- & -- \\
\bottomrule
\end{tabular}
}
\end{table*}

\end{document}